\documentclass{sig-alternate}

\newcommand{\R}{\mathcal{R}}
\newcommand{\V}{\mathcal{V}}
\newcommand{\E}{\mathcal{E}}
\newcommand{\G}{\mathcal{G}}
\newcommand{\N}{\mathcal{N}}

\newcommand{\remove}[1]{}

\newtheorem{corollary}{\textbf{Corollary}}
\newtheorem{lemma}{\textbf{Lemma}}

\newtheorem{theorem}{\textbf{Theorem}}
\usepackage{latexsym}
\usepackage{booktabs}
\usepackage{epstopdf}
\usepackage{pgfplots}
\usepackage{url}
\usepackage{booktabs}
\usepackage{multirow}

\begin{document}
%

\title{Randomized Kaczmarz for Rank Aggregation from Pairwise Comparisons
}
%
%
%
%
%

%
\author{
%
%
Vivek S.\ Borkar \hspace{.75in} Nikhil Karamchandani \hspace{.75in} Sharad Mirani \\[1em]
Department of Electrical Engineering, IIT Bombay, Mumbai, India\\[.75em]
Emails: borkar.vs@gmail.com, nikhilk@ee.iitb.ac.in, sharad.mirani@iitb.ac.in
}
      \remove{
       \affaddr{Department of Electrical Engineering}\\
       \affaddr{IIT Bombay}\\
       \affaddr{Powai, Mumbai 400076, India}\\
       \email{borkar.vs@gmail.com}
\alignauthor
Nikhil Karamchandani\\
       \affaddr{Department of Electrical Engineering}\\
       \affaddr{IIT Bombay}\\
       \affaddr{Powai, Mumbai 400076, India}\\
       \email{nikhilk@ee.iitb.ac.in}
\and
\alignauthor
Sharad Mirani\\
       \affaddr{Department of Electrical Engineering}\\
       \affaddr{IIT Bombay}\\
       \affaddr{Powai, Mumbai 400076, India}\\
       \email{sharad.mirani@iitb.ac.in}
}


\maketitle
\begin{abstract}
We revisit the problem of inferring the overall ranking among entities in the framework of Bradley-Terry-Luce (BTL) model, based on available empirical data on pairwise preferences. By a simple transformation, we can  cast the problem as that of solving a noisy linear system, for which a ready algorithm is available in the form of the randomized Kaczmarz method. This scheme is provably convergent, has excellent empirical performance, and is amenable to on-line, distributed and asynchronous variants. Convergence, convergence rate, and error analysis of the proposed algorithm are presented and several numerical experiments are conducted whose results validate our theoretical findings.

\end{abstract}

\remove{


}
\section{Introduction}
Rank aggregation is the problem of combining multiple (partial) preferences over a collection of items into a single `consensus' ordering that best describes the available data. It finds applications in a wide variety of domains, ranging from web search \cite{BrinPage, dwork} to recommendation systems \cite{oh2015collaboratively}, and from competitive sports and online gaming systems \cite{trueskill} to crowdsourced services \cite{Chen}.

One particular category of data which is quite popular in the literature is \textit{pairwise comparisons}, for example a recommendation system enquiring which of a pair of items does a user prefer, or the result of a match between two chess players. Results of such pairwise comparisons can be used to estimate the inherent `quality' or `score' of an item, for example the skill level of a chess player, and can be modeled as noisy samples of the relative score of the items being compared. By comparing several item pairs repeatedly, one can estimate the inherent scores of the various items and in turn, use it to decide on a ranking of the items. This is the context in which this paper is placed.

In particular, we consider the popular Bradley-Terry-Luce (BTL) model \cite{BradleyTerry, Luce} for pairwise comparisons and using a simple transformation, convert the problem of inferring the item values into one of solving a noisy linear system of equations. We employ a randomized version of the widely popular Kaczmarz method \cite{Strohmer} for solving this system and present an analysis of the resulting error, in terms of the spectral properties of the underlying comparison graph. This allows us to characterize the number of pairwise comparisons needed to achieve a certain error threshold. We find that for the case where the comparison graph is an Erd\H{o}s-R\'enyi graph, i.e., item pairs are chosen uniformly at random for comparison, the total number of comparisons needed by our scheme is in fact order-optimal. We discuss online, distributed, and asynchronous variants of the scheme                                                                                                                                                                                                                                                                                                                                                                                                                                                                                                                                                                                                                                                                                                                                                                                                                                                                                                                                                                                                                                                                                                                                                                                                                                                                                                                                                                                                                                                                                                                                                                                                                                                                                                                                                                                                                                                                                                                                                                                                                                                                                                                                                                                                                                                                                                                                                                                                                                                                                                                                                                                                                                                                                                                                                                                                                                                                                                                                                                                                                                                                                                                                                                                                                                                                                                                                                                                                                                                                                                                                                                                                                                                                                                                                                                                                                                                                                                                                                                                                                                                                                                                                                                                                                                                                                                                                                                                                                                                                                                                                                                                                                                                                                                                                                                                                                                                                                                                                                                                                                                                                                                                                                                                                                                                                                                                                                                                                                                                                                                                                                                      and  run extensive numerical experiments to validate our theoretical findings.

\subsection{Related work}
There is a vast literature on rank aggregation, we only discuss the works that we feel are the most relevant to the contents of this paper. The main theme of this paper is to infer a ranking over a collection of items from noisy data, generated according to a statistical model. There is a wide variety of such probabilistic models studied in the literature, see for example \cite{qin2010new, probmodel}. \cite{braverman2009sorting, wauthier2013efficient} study the problem of ranking with noisy comparisons between item pairs, where the result of each comparison follows the true order with probability $p$ for some $p > 1/2$. Other variants include active ranking  \cite{jamieson2011active}, where the items to be compared are chosen in an adaptive and sequential fashion, and adversarial comparators \cite{acharya2014}. Another popular model is the Mallows model, which given a true ranking $\sigma^*$, randomly generates a noisy full ranking $\sigma$ with probability proportional to $\exp(-\beta d(\sigma, \sigma^*))$, where $\beta$ is a spreading parameter and $d(\cdot, \cdot)$ is a distance metric over permutations, such as the Kendall-Tau or the Kemeny distance.  \cite{braverman2009sorting} present polynomial time algorithms for identifying the true ranking over $n$ items with high probability, given $O(\log n)$ independent noisy rankings. Random Utility Models (RUMs) \cite{thurstone} present another alternative, where each item $i$ is associated with a score $w_i$. An instance of the available noisy data is a (possibly partial) ranking $\sigma$ generated by assigning a random utility $X_i$ for each item $i$, according to a conditional distribution $\mu_{i}(\cdot | w_i)$, and then ordering them. A special case of RUMs is the Plackett-Luce (PL) \cite{Luce, plackett}, where the random utilities are generated according to  Gumbel distributions. The PL model allows for an analytical characterization of the Maximum Likelihood Estimator \cite{hunter2004mm, maystre2015fast} and the optimal number of independent partial rankings required to achieve a target error \cite{hajek2014minimax}. In this work, we focus on the Bradley-Terry-Luce (BTL) model \cite{BradleyTerry, Luce}, which is a special case of the PL model where only pairwise comparisons are allowed. Rank aggregation under the BTL model has received a lot of attention recently \cite{suhspectral, shah2015estimation, rajkumar2014statistical}. The work closest to ours is \cite{negahban2012rank} which proposes an iterative algorithm called Rank Centrality for estimating the underlying item scores. The algorithm is based on the Markov Chain Monte Carlo (MCMC) method, with the transition matrix constructed  using results of various pairwise comparisons, and the score estimate vector being the leading eigenvector. In contrast, we formulate the problem as one of solving a noisy system of linear equations and use the randomized version of the iterative Kaczmarz solution method which is provably convergent, has excellent empirical performance, and is amenable to on-line, distributed and asynchronous variants. In spirit, our work is also close to \cite{stefani1977football, stefani1980improved, jiang2011statistical, hirani2010least} which pose rank aggregation as a least squares problem.

While most of the literature mentioned above considers the case of one true ranking, there has been recent work on collaborative ranking for a pool of users \cite{lu2014individualized, liu2009probabilistic, park2015preference, oh2015collaboratively, Wu2015}. Finally, unlike the works mentioned above, a non-parametric model for distributions over rankings has been proposed in \cite{jagabathula2008inferring}.
\remove{
\begin{itemize}
\item Probabilistic models on permutations: \cite{qin2010new, probmodel},
\item Non-parameteric model of distributions over permutations: \cite{jagabathula2008inferring}
\item Parametric model, noiseless comparisons: \cite{jamieson2011active}
\item Parametric model, noisy comparisons : \cite{braverman2009sorting, wauthier2013efficient}
\item Parametric model, distance based: Mallows model, \cite{braverman2009sorting}
\item Parametric model, Random Utility model, or Thurstone model: Partial Rankings \cite{hajek2014minimax}, Plackett-Luce \cite{hunter2004mm}, Bradley-Terry (special case of PL, where only pairwise comparisons) \cite{negahban2012rank, suhspectral} Topology bound for Thurstone and BTL \cite{shah2015estimation}, statistical convergence properties \cite{rajkumar2014statistical}
\item Many rankings (preferences) from comparisons: \cite{lu2014individualized, liu2009probabilistic, park2015preference, oh2015collaboratively}
\end{itemize}
}

\section{The problem and the algorithm}
\label{Sec:Problem}
	
\ \\
We consider $N >>1$ entities identified with the nodes of an undirected graph $\G = (\V, \E)$, where $\V$ is its node set (thus $|\V| = N$) and $\E$ its edge set, with $|\E| = M$ (say). We assume that the graph is connected. Following the Bradley-Terry-Luce (BTL) model, we postulate `node weights' $w_i > 0$ associated with node $i \in \V$. Let $[\underline{w}, \bar{w}]$ denote the dynamic range of the $w_i$'s and $b = \bar{w} / \underline{w}$. Set $p_{ij} \ (:=$ the probability that $i$ is preferred over $j) = \frac{w_i}{w_i + w_j}$. Given $(i, j) \in \E$, let the outcome $X_{ij}^{k}$ of the $k$-th comparison between $i$ and $j$ be defined as $1$ if $i$ is preferred over $j$, and $0$ otherwise. Then, according to the BTL model,
\begin{align}
X_{ij}^{k} &=
\begin{cases}
1, \mbox{ with probability } p_{ij} \\
0, \mbox{ otherwise}
\end{cases}
\label{Eqn:BTL}
\end{align}
For each $(i, j) \in \E$, we will in general assume that multiple such comparisons are made, and the corresponding outcomes are assumed to be independent across $i, j$, and $k$. Thus what we have are estimates $\hat{p}_{ij}$ of $p_{ij}$'s, viz.,
\begin{equation}
\label{Eqn:ProbEstimate}
\hat{p}_{ij} := \frac{\sum_kX^k_{ij}}{\sum_kX^k_{ij} + \sum_kX^k_{ji}},
 \end{equation}
the fraction of times $i$ was preferred over $j$. The nodes are to be ranked according to the decreasing values of $w_{\cdot}$, based on estimates thereof. These  have to be computed from available data regarding observed preferences of a population that gives pairwise preferences among neighboring nodes of $\G$ (rather, we consider a pair of nodes neighbors when such data is available for them). Thus
\begin{eqnarray*}
\hat{p}_{ij} \approx p_{ij} &=& \frac{w_i}{w_i + w_j} \\
\Longrightarrow \frac{w_j}{w_i} &=& \frac{1}{p_{ij}} - 1 \approx \frac{1}{\hat{p}_{ij}} - 1 \\
\Longrightarrow \log w_i - \log w_j &=& - \log\left(\frac{1}{p_{ij}} - 1\right) \approx - \log\left(\frac{1}{\hat{p}_{ij}} - 1\right).
\end{eqnarray*}
Set $v_i := \log w_i, \ i \in \V, \ y'_{ij} := - \log\left(\frac{1}{p_{ij}} - 1\right)$.  Let $v = [v_1, \cdots, v_N]^T$ and $y' :=$ the vector of $y'_{ij}$'s lexicographically arranged, after retaining only one of the pair $y'_{ij}, y'_{ji}$ for each $(i,j)$, say the smaller one if they are unequal and either one if they are equal. This removes redundancy, since $p_{ij} + p_{ji} = 1$, leading to
$$y'_{ij} = -\log\left(\frac{1}{1 - \frac{1}{1 + e^{-y'_{ji}}}} - 1\right).$$
  We retain the edge $(i, j)$  and drop the edge $(j,i)$ if $y'_{ij}$ is retained. The graph is now directed with the same node set as before. We continue to refer to it as $\G = (\V, \E)$  by abuse of notation.  The presence of an edge $(i,j)$ now means that $j$ is preferred over $i$ in at least half the samples. We assign a direction to the edge $(i,j)$ from $i$ to $j$ if $y'_{ij}$ is retained and $j$ to $i$ if not. Denote by $L \in \R^{N\times M}$ the incidence matrix associated with the graph, i.e., the node-edge matrix such that if we consider the column, say $l$, corresponding to edge $(i,j)$ with direction from $i$ to $j$, $(i, l)$ is $1$, $(j, l)$ is -1, and all other elements are $0$. Then  we can cast the above relationship as $y' = L^Tv$. We do not, however, have access to $y'$. What we have instead is $y :=$ the vector of $y_{ij}$'s, where $y_{ij} := - \log\left(\frac{1}{\hat{p}_{ij}} - 1\right)$. Thus what we have is
$y = L^Tv \ +$ noise. Our problem then is to estimate $v$. Casting it as the problem of minimizing the quadratic error criterion $\|y - L^Tv\|^2$ over $v$  leads to the optimality equation
\begin{equation}
Ly = LL^T\hat{v}, \label{linear}
\end{equation}
where $LL^T$ is in fact the Laplacian matrix for the graph $\G$ and $\hat{v}$ denotes the desired estimate. Our problem has now been reduced to that of solving a noisy linear system of equations. Note that we have an underdetermined system of equations, since the Laplacian matrix $LL^T$ is rank deficient. In fact, the eigenvector corresponding to eigenvalue $0$ is the all-one vector and so $\hat{v}$ can be determined only up to a constant-vector shift. For solving a linear system of equations, a randomized version of Kaczmarz algorithm \cite{Strohmer} can be used. We describe this next.  

The $(i,i)$th diagonal element of $LL^T$ is $N(i) :=$ the total degree (in-degree + out-degree) of node $i$. For $j \neq i$, the $(i,j)$th element of $LL^T$ is $-1$ if $i,j$ are neighbors, $0$ otherwise. Let $a_i := $ the $i$th row of $LL^T$ and $b = Ly$. Then
$$\|a_i\| = \sqrt{N(i)(N(i) + 1)}.$$
Let $\N(i)$ denote the set of neighbours of $i$. The randomized Kaczmarz algorithm for solving a system of linear equations $Ax = b$ is given by 
\begin{eqnarray}
\lefteqn{x(n+1) = x(n) \ +} \nonumber \\
&& \sum_iI\{\xi_n = i\}\left(\frac{b_i - \langle a_i, x(n)\rangle}{\|a_i\|^2}a_i^T\right) .\label{Kac0}
\end{eqnarray}
Here $\{\xi_n\}$ are IID random variables taking values in the set $\{1, 2, \cdots, N\}$ with $p_i := P(\xi_m = i) > 0 \ \forall i$. In the present set-up, this translates into
\begin{eqnarray}
\lefteqn{x(n+1) = x(n) + \sum_iI\{\xi_n = i\}\times \frac{1}{\|a_i\|^2}\times} \nonumber \\
&&\left(\sum_{\{j : (i,j) \in \E\}}(y_{ij} - y_{ji}) - (x_i(n) - x_j(n))\right)a_i^T \nonumber \\
\ && \ \label{Kac}
\end{eqnarray}
where $x(n)$ corresponds to the estimate for $\hat{v}$ at the $n^{th}$ iteration. Recall that for each pair $i, j \in \V$, at most one of $y_{ij}, y_{ji}$ is non-zero. \\

Here the idea is to update one component of the iteration at a time and $\xi_n :=$ the index of the component chosen at time $n$. In classical Kaczmarz scheme, $\xi_n$ is periodic in a round robin manner. We stick to the randomized scheme in view of the proven performance gains for it over the classical set-up \cite{Strohmer}, \cite{Freris}, and its better adaptability for on-line scheme that we describe later. 

\section{Convergence of the algorithm}
\ \\
We now discuss convergence and convergence rate for a general randomized Kaczmarz scheme, with the ranking problem considered in this paper being a special case. e consider a linear system \begin{equation}
Ax = b \label{system}
\end{equation}
which may be underdetermined (as in our case), exactly determined, or overdetermined and consistent. The general randomized Kaczmarz scheme is
\begin{equation}
x(n+1) = x(n) + \sum_iI\{\xi(n)= i\}\left(\frac{b_i - \langle a_i, x(n)\rangle}{\|a_i\|^2}\right)a_i^T, \ \label{eqn-i}
\end{equation}
for $1 \leq i \leq N$. Introduce the notation
\begin{eqnarray*}
\check{b}_i &:=&\frac{b_i}{\|a_i\|}, \ \check{b} := [\tilde{b}_1, \cdots, \tilde{b}_N]^T, \\
\check{a}_i &:=& \left(\frac{1}{\|a_i\|}\right)a_i, \\
\textbf{1} &=& [1,1,\cdots, 1]^T, \theta := [0,0, \cdots, 0]^T, \\
H &:=& \left\{r : z^Tr = z^Tx(0), \ \forall \ z \ \mbox{such that} \ Az = \theta\right\} \\
H_0 &:=& \left\{r : z^Tr = 0,  \ \forall \ z \ \mbox{such that} \ Az = \theta\right\}
\end{eqnarray*}
For the problem studied in this paper, $A = LL^T$ and we have $Az = \theta$ only for $z = \textbf{1}$. So $H$ consists of all vectors $y$ such that $\sum_{i} r_i = \sum_{i} x_i(0)$.\\

For any $z$ such that $Az = \theta$, we have from \eqref{eqn-i} that
\begin{eqnarray*}
z^Tx(n+1) &=& z^Tx(n) + \theta = z^Tx(n).
\end{eqnarray*}
Hence the iterates in the randomized Kaczmarz scheme always remain in the affine space $H$. ($H$ is the whole space for exactly determined and consistent overdetermined $A$.)\\

We have the following simple lemma.
\begin{lemma}
There is a unique solution $x^*$ in $H$ to $Ax = b$.
\end{lemma}
\begin{proof}
Suppose there are two distinct solutions $x_1, x_2$ in $H$ to $Ax = b$.  Then, we have $A(x_1 - x_2) = \theta$. From the definition of $H$, we must have
$$
(x_1 - x_2)^T x_1 = (x_1 - x_2)^Tx_2 \Longrightarrow \| (x_1 - x_2) \|^2 = 0,
$$
which is a contradiction, since we assumed $x_1 \neq x_2$.
\end{proof}

Define
\begin{eqnarray*}
e(n) &:=& x(n) - x^*, \ \check{e}(n) := \frac{e(n)}{\|e(n)\|}, \\
s^* &:=& \underset{\{s : \|s\| = 1, s \in H_0\}}{\operatorname{argmin}}\left(\sum_ip_i|\langle \check{a}_i, s\rangle|^2\right).
\end{eqnarray*}
Thus, $s^*$ is the  eigenvector of the non-negative definite matrix $S :=\sum_ip_i\check{a}_i^T\check{a}_i$ corresponding to the minimum non-zero eigenvalue $\lambda_{min}$ of $S$. Further, the minimum value of the quantity being minimized is in fact $\lambda_{min}$. Convergence of the randomized Kaczmarz scheme has been widely studied, see for example \cite{Strohmer, Deanna1, Deanna2, gower2015randomized}. The following result establishes convergence and provides a lower bound on the rate of convergence. \\
\begin{theorem}
\label{Thm:RateK}
Suppose $\lambda_{min} \in (0, 1)$. Then $e(n) \to 0$ a.s.\ and $E\left[\|e(n)\|^2\right] \to 0$ exponentially. In particular, 
$$
E\left[\|e(n)\|^2\right] \le (1 - \lambda_{min})^n \cdot E\left[\|e(0)\|^2\right].
$$
\end{theorem}

The proof of the above result has been included in the appendix for completeness. \cite{Strohmer} proposed the choice $p_i = \frac{\|a_i\|^2}{\sum_j\|a_j\|^2}$ and the following argument shows that $\lambda_{min} \in (0, 1)$ for this choice. Let $\| \cdot  \|_F$ denote  the Frobenius  norm of a non-negative definite matrix. Let $\lambda_{min}^*$ denote the minimum non-zero eigenvalue of $A^TA$. With above choice of $p_i$, we have
\begin{eqnarray*}
\sum_ip_i\check{a}_i^T\check{a}_i &=& \frac{1}{\|A\|_F^2}\sum_ia_i^Ta_i \\
\Longrightarrow \sum_ip_i|\langle\check{a}_i, s^*\rangle|^2 &=& \lambda_{min} \\
&=& \frac{1}{\|A\|_F^2}\lambda_{min}^* \\
&=& \frac{\lambda_{min}^*}{\mbox{tr}(A^TA)} \ <  \ 1.
\end{eqnarray*}
Thus by choosing $p_i = \frac{\|a_i\|^2}{\sum_j\|a_j\|^2}$, we are guaranteed exponential convergence for the randomized Kaczmarz algorithm. Finally, specializing the result to our problem, we have:\\

\begin{corollary}
Almost surely, $x(n) \to x^* :=$ the unique solution to (\ref{system}) satisfying $\sum_ix^*_i = \sum_ix_i(0)$ and $$E\left[\|x(n) - x^*\|^2\right] \to 0$$ at an exponential rate.
\end{corollary}

We now briefly comment on the complexity of the proposed scheme. In each iteration, we need to calculate the $(\sum (y_{ij} - y_{ji})$ $- (x_i(n) - x_j(n)))$ only for the neighbours of node $i$. In terms of time complexity, this take $O(N(i))$ number of computations. Also, for each iteration we update only the node chosen in that iteration along with its neighbors $= N(i) + 1$ =  $O(N(i))$ (non-zero entries in the row $a_i$). Thus the total number of computations per iteration is $S \propto O(N(i))$. 
In the case of the randomized Kaczmarz algorithm, where we choose the node $i$ with probability $p(i)\propto ||a(i)||^2 = N(i)^2 + N(i)$, the expected number of computations per iteration is given by
$$\mathbb{E}[S] = \mathbb{E}[p(i)\cdot N(i)] + \mathbb{E}[1],$$
$$p(i)\cdot N(i) = \frac{N(i)^2(N(i)+1)}{\sum_j N(j)(N(j)+1)}$$
$$\Longrightarrow \mathbb{E}[p(i)\cdot N(i)] = \mathbb{E}\left[\frac{N(i)^2 + N(i)^3}{N(i) + N(i)^2}\right] = \mathbb{E}\left[N(i)\right].$$
For the special case when the underlying graph $\G$ is an Erd\H{o}s-R\'enyi graph with edge probability $p$, $\mathbb{E}[N(i)] = (N -1)p$ and hence $\mathbb{E}[S] = O(Np)$.

Let $k_\epsilon$ give us the number of iterations required to reach within an error $\epsilon$ of the solution. From \cite{Strohmer}, we can see that expected value of $k_\epsilon$ is given as:
$$\mathbb{E}[k_\epsilon] \leq \frac{2\log \epsilon}{\log (1 - \lambda^*_{min}/trace(A^TA))} \approx \frac{trace(A^TA)}{\lambda^*_{min}}\log \frac{1}{\epsilon}.$$
For our setup, the matrix $A$ is the Laplacian matrix $LL^T$ of the underlying comparison graph $\G$. For the special case when the underlying graph $\G$ is an Erd\H{o}s-R\'enyi graph with edge probability $p$, all the eigenvalues of the matrix $A^TA$ are $\Theta((Np)^2)$ \cite{mohar1993eigenvalues, hoffman2012spectral}, $\frac{trace(A^TA)}{\lambda^*_{min}} \approx N$ as $N$ grows large. Hence $$\mathbb{E}[k_\epsilon] = O(N)$$ and the total number of computations $T$ for our algorithm is given by
$$\mathbb{E}[T] = \mathbb{E}[S]\cdot \mathbb{E}[k_\epsilon] = O(N^2p).$$ 

\section{Error Analysis}
\label{Sec:ErrorAnalysis}
\ \\
In this section, we consider the error performance of our proposed scheme for ranking using pairwise comparisons and have the following main result:
\begin{theorem}
\label{Thm:Error}
Consider $N$ entities with associated weights $w_1, w_2, \ldots, w_N$ and a connected comparison graph $\G = (\V, \E)$ with $| \E | = M$. Let $C$ denote the total number of comparisons, such that each pair $(i, j) \in \E$ is compared $k = C / M$ times, with outcomes according to the BTL model~\eqref{Eqn:BTL}. Then for $k \ge \Omega(\log N)$, the normalized weight error of the proposed scheme using the randomized Kaczmarz algorithm is given by
 \begin{eqnarray*}
   \frac{ \|\widehat{w} -  w\|  }{\|w\|} 
   &\leq& O\left(\frac{Mb(1 + b)\sqrt{\lambda^L_{max}\log M}}{\sqrt{CN}\lambda^L_{min}}\right)
  \end{eqnarray*}
with high probability (w.h.p), where $\lambda_{max}^L, \lambda_{min}^L$ denote the maximum and minimum eigenvalues respectively of the Laplacian matrix for the comparison graph $\G$; and $[\underline{w}, \bar{w}]$ denotes the dynamic range of the $w_i$'s with $b = \bar{w} / \underline{w}$.
\end{theorem}
\begin{proof}
The first part of the proof proceeds through a sequence of steps in order to characterize the error in estimating $v$ by solving the set of linear equations in \eqref{linear}. Throughout, we use the notation `$\Delta\cdots$' for error in `$\cdots$'.  

 \begin{itemize}

\item \textit{Claim 1: $\| \Delta v \| \le O(1 / \lambda_{min}^{L})\|\Delta Ly'\|$} 
\begin{proof}From \eqref{linear}, we compute an estimate $\hat{v} = v + \Delta v$ by solving a noisy version of the linear system $Ly' = LL^T v$, restricted to a translation of the orthogonal compliment of the null space of $LL^T$. Thus any error $\|\Delta Ly'\|$ in $Ly'$ will lead to an error of at most $O(1 / \lambda_{min}^{L})\|\Delta Ly'\|$ in our estimate of $v$, where $\lambda_{min}^L$ is the minimum non-zero eigenvalue of the Laplacian matrix $LL^T$ for the underlying graph $\G$.
\end{proof}
\item \textit{Claim 2: $\| \Delta Ly' \| \le O(\sqrt{\lambda^L_{max}}) \|\Delta y'\|$} 
\begin{proof}
Any error $\|\Delta y'\|$ in $y'$ will lead to an error of at most $O(\sqrt{\lambda^L_{max}}) \|\Delta y'\|$ in $\Delta Ly'$, where $\lambda_{max}^L$ is the maximum eigenvalue\footnote{This uses the fact that non-zero eigenvalues of $LL^T$ and $L^TL$ are identical.} of the Laplacian matrix $LL^T$ for the underlying graph $\G$.
\end{proof}
 \item \textit{Claim 3: For each $(i, j) \in \E$, \\
 $\Delta y'_{ij} \le$ $2(1+b)\left(|\Delta p_{ij}| + |\Delta p_{ji}|\right)$ w.h.p.} 
 \begin{proof}
 Recall that
 \begin{eqnarray*}
 y'_{ij} &=&-\log\left(\frac{1}{p_{ij}} - 1\right), \\
 y_{ij} = y'_{ij} + \Delta y'_{ij} &=& -\log\left(\frac{1}{\hat{p}_{ij}} - 1\right).
 \end{eqnarray*}
 Then we have
 $$
 y'_{ij} = \log(p_{ij}) - \log(p_{ji})
 $$
 and
\begin{eqnarray*}
y_{ij}
&=& \log\left(\frac{1}{k}\sum_lX^l_{ij}\right) - \log\left(\frac{1}{k}\sum_lX^l_{ji}\right) \\
&=& \log(p_{ij} + \Delta p_{ij}) - \log(p_{ji} + \Delta p_{ji}).
\end{eqnarray*}
By the mean value theorem, there exists $p^* \in (p_{ij} - | \Delta p_{ij} |, p_{ij} + | \Delta p_{ij} |)$ such that
\begin{align*}
\frac{ \log(p_{ij} + \Delta p_{ij}) - \log(p_{ij})  }{ \Delta p_{ij} } &=& \frac{d \log p}{dp}  \ \Big \rvert_{p = p^*} \\
&=& \frac{1}{p^*} \\
&\ge& \frac{1}{p_{ij} - | \Delta p_{ij} |}\\
&\overset{(a)}{\ge}& \frac{2}{p_{ij}}\\
&\ge& 2(1 + b)
\end{align*}
where $(a)$ follows w.h.p. from \eqref{Eqn:ProbEstError2} below and the last inequality holds since $p_{ij} \ge 1/(1+b)$. Thus we have
\begin{align*}
| \log(p_{ij} + \Delta p_{ij}) - \log(p_{ij})  | \le 2(1+b) | \Delta p_{ij} |.
\end{align*}
Similarly, we can show that
\begin{align*}
| \log(p_{ji} + \Delta p_{ji}) - \log(p_{ji})  | \le 2(1+b) | \Delta p_{ji} |.
\end{align*}
Combining the above inequalities, we have
$$
|\Delta y'_{ij}| \leq 2(1+b)\left(|\Delta p_{ij}| + |\Delta p_{ji}|\right).
$$
\end{proof}
\remove{
Next, we have
$$
\frac{d \log p}{dp}  \ \Big \rvert_{p = p_{ij}} = \frac{1}{p_{ij}} = \frac{w_i + w_j}{w_{i}} \le 1 + \frac{\bar{w}}{\underline{w}} = 1 + b,
$$
and hence by the mean value theorem, we have
$$
 \log(p_{ij} + \Delta p_{ij}) - \log(p_{ij}) \le
$$
$$|\Delta y'_{ij}| \leq \frac{1}{\underline{w}}(|\Delta p_{ij}| + |\Delta p_{ji}|).$$
}
 \item \textit{Claim 4: For $\eta > 0$ and each $(i, j) \in \E$, $P(|\Delta p_{ij}| \geq \eta) \leq 2e^{-2\eta^2k}$} 
 \begin{proof}
 We have $k$ comparisons between $i$ and $j$. Then from \eqref{Eqn:ProbEstimate}, we have an estimate $ \hat{p}_{ij} = (p_{ij} + \Delta p_{ij})$ of $p_{ij}$ based on these measurements, given by
 $$
\hat{p}_{ij} := \frac{\sum_lX^l_{ij}}{\sum_lX^l_{ij} + \sum_lX^l_{ji}} = \frac{\sum_lX^l_{ij}}{k}.
 $$
It follows from Hoeffding inequality that for any $\eta > 0$,
\begin{equation}
\label{Eqn:ProbEstError}
P(|\Delta p_{ij}| \geq \eta) \leq 2e^{-2\eta^2k} 
\end{equation}
 which proves the claim. In particular, if we set $\eta = p_{ij} / 2$ and $k = 6\log N / p_{ij}^2$, then we have
\begin{equation}
\label{Eqn:ProbEstError2}
|\Delta p_{ij}| \leq p_{ij} / 2 \ \mbox{with prob.} \ 1 - O(1 / N^3).
\end{equation}
Since $p_{ij} = w_i / (w_i + w_j) \ge 1/(1 + b)$ for all edges $(i, j)$ and there are at most $O(N^2)$ edges in the graph, $k \ge 2(1+b)^2\log N \ \forall (i, j)$ suffices for the above bound on $|\Delta p_{ij}|$ to hold true for all $(i, j) \in \E$ w.h.p. as $N$ grows large.
\end{proof}
\end{itemize}

Combining all the preceding claims, we then have that the total error in the estimate $\hat{v}$ is given by
    \begin{align*}
    &&\|\Delta v\| \leq O\left(\frac{\eta(1+b)\sqrt{\lambda^L_{max}M}}{\lambda^L_{min}}\right) \\
     &&\mbox{with probability} \ 1 - 2Me^{-2\eta^2C/M}.
     \end{align*}
Taking $\eta = \sqrt{M\log M/C}$, we have
    \begin{align*}
    &&\|\Delta v\| \leq O\left(\frac{M(1 + b)\sqrt{\lambda^L_{max}\log M}}{\sqrt{C}\lambda^L_{min}}\right)  \\
    &&\mbox{with probability} \ 1 - O(1/M).
    \end{align*}
    \noindent Then
$$(\hat{w}_i - w_i)^2 = w_i^2(\exp(\Delta v) - 1)^2 \approx w_i^2\cdot \Delta v^2$$
$$\Rightarrow ||\hat{w} - w|| \leq \bar{w} ||\Delta v||$$
$$\Rightarrow \frac{||\hat{w} - w||}{||w||} \leq \frac{\bar{w} ||\Delta v||}{\underline{w} \sqrt{N}}$$
Thus, as $N$ grows large, the above inequality holds w.h.p.. In that case, the normalized weight error is  given by

 \begin{eqnarray}
   \frac{ \|\Delta w\|  }{\|w\|} &\leq&  \frac{b\cdot \|\Delta v\|}{\sqrt{N}} \nonumber \\
   &\leq& O\left(\frac{Mb(1 + b)\sqrt{\lambda^L_{max}\log M}}{\sqrt{CN}\lambda^L_{min}}\right). \label{doubleu}
  \end{eqnarray}
  which completes the proof of Theorem~\ref{Thm:Error}.
\end{proof}
\remove{
Thus, in order to ensure that the normalized weight error is at most $\epsilon$, the total number of comparisons $C$ needed is given by
\begin{equation}
\label{Eqn:Comparisons}
C \ge \Omega\left(\frac{M^2 b^2 (1 + b)^2\lambda^L_{max}\log M}{N(\lambda^L_{min})^2}\right).
\end{equation}
}
A couple of comments are in order.
\begin{enumerate}
\item In the calculations above, we have not accounted for the additional error due to the finite run of the randomized Kaczmarz scheme. From Theorem 1, we have an exponential bound on the mean square error caused thereby. Specifically, after $n$ iterations of the randomized Kaczmarz scheme, the additional mean square error is
$O\left( \alpha^n \right)$ for some $\alpha \in (0, 1)$, which converges to zero exponentially fast.

\item The estimate $\hat{p}_{ij}$ of $p_{ij}$ is based on the strong law of large numbers and is unbiased. But that is not so for the estimate $y_{ij}$ of $y'_{ij}$ because of the intervening nonlinear transformations. Nevertheless, since our problem of ranking is an ordinal problem that is insensitive to sufficiently small errors, the foregoing ensures correct ranking with a very high probability if sufficiently many samples are used for estimating the probabilities and then the randomized Kaczmarz is run for sufficiently long.\\ 

\item Consider the special case when the comparison graph $\G$ is the Erd\H{o}s-R\'enyi graph, so that for each pair of nodes, the edge between them exists with some probability $p$. For $p \ge \Omega(\log N / N)$, which is the minimum needed to ensure that the graph is connected as the size of the graph $N$ grows large, we have the number of edges $M = \Theta(N^2p)$, and  both $\lambda^L_{max}$ and $\lambda^L_{min}$ are $\Theta(Np)$ \cite{mohar1993eigenvalues, hoffman2012spectral}. From Thoerem~\ref{Thm:Error} and $C = Mk$, we have 
\begin{eqnarray}
\nonumber
 \frac{ \|\Delta w\|  }{\|w\|} &\leq& O\left(\frac{N^2pb(1 + b)\sqrt{Np \log (N^2p)}}{\sqrt{kN^2p \cdot N}\cdot Np}\right) \\
 &\leq& O\left(\frac{b(1 + b)\sqrt{\log N}}{\sqrt{k}}\right)
 \label{Eqn:ERerror}
\end{eqnarray}
Thus, in order to ensure that the normalized weight error is at most some constant $\epsilon > 0$, we need $k \ge \Omega(\log N)$. When $p = \Theta(\log N / N)$, the number of edges $M$ is $\Theta(N\log N)$ w.h.p. as $N$ grows large, and hence the total number of comparisons needed is $C \ge \Omega\left(N\log^2N\right)$. The minimum number of edges needed to ensure w.h.p. that the Erd\H{o}s-R\'enyi graph is connected is $\Omega(N\log N)$, so the above requirement on $C$ is optimal upto logarithmic factors. This is similar to the result obtained in \cite{negahban2012rank}.
\remove{
\item For the Erdos-Renyi model, we can also estimate the dependence on the average degree $d = Np$ for the above set up, i.e., with $p \approx \log N / N$.  Recall also that we required $m_{ij} \geq 2(1 + b)\log N$, leading to $C = \Theta\left(N\log N \cdot \log N\right)$. 
    Combining the estimates with (\ref{doubleu}), we obtain
    \begin{eqnarray*}
    \frac{\|\Delta w\|}{\|w\|} &\leq& O\left(\frac{M\sqrt{\lambda^L_{max}\log M}}{\sqrt{CN}\lambda^L_{min}}\right) \\
    &=& O\left(\frac{N\log N\sqrt{\log N \cdot \log N}}{\sqrt{N\log N\cdot \log N \cdot N}\log N}\right)  \\
    &=&  O(1),
    \end{eqnarray*}
i.e., the error remains bounded. 
}


\end{enumerate}

\section{Remarks and Extensions}
\ \\
We sketch here several important variants and extensions, along with some general remarks.

\begin{enumerate}
\item \textit{\large Optimal sampling distribution:}\\

 Let
$$F(p) := \sum_ip_i|\langle\check{a}_i, z^*(p)\rangle|^2.$$
We can  define the optimal sampling distribution as
$$
p^* = [p^*_1, \cdots, p^*_N] \in \mbox{Argmax}_pF(p).
$$
This leads to the problem of evaluating the outer maximizer of
$$
\max_p\min_{\{z : \|z\| = 1, z \in H_0\}}\sum_ip_i|\langle\check{a}_i, z\rangle|^2.
$$
Note that the problem is not amenable to the Von Neumann - Ky Fan minmax theorem because the inner minimization is over a sphere, a non-convex set.\\

\item \textit{\large Exactly determined system:}\\

 As already observed, $v$ can be specified only up to an additive scalar, since its pairwise differences is the only thing that counts. Thus we may set one component of $v$, say $v_{i_0}$, equal to zero, which is tantamount to dropping the corresponding row and column of $L$ from consideration.  This modification renders $L, LL^T$ full rank. We also experimented with the randomized Kaczmarz corresponding to this exactly determined system, but the performance was not as good as the underdetermined system. \\

\item \textit{\large Comparison with Clock Syncronization:}\\

It is also worth noting that the equations we have are exactly the same as those arising in clock synchronization where similar issues arise \cite{Solis}. The algorithm proposed in \cite{Solis} is another alternative scheme which is quite similar to ours. Our experimentation, however, indicated that the present randomized Kaczmarz scheme has a superior performance.\\

\item \textit{\large Ranking based on insufficient data:}\\

This corresponds to the case when we have data only on a small subset of edges, so that $L$ does not correspond to a connected graph provided in advance but only to a subset of its edges. The Kaczmarz algorithm works nevertheless in view of our analysis above, the only difference being that $H$ is now a higher dimensional space. The iterates then converge a.s.\ to an initial condition dependent point in $H$ as proved above.\\

\item \textit{\large On-line distributed scheme:}\\

Suppose that the user preference data is episodic and we correspondingly keep running estimates of $\{p_{ij}\}$, updating each when a new observation relevant to the particular estimate appears. The randomized Kaczmarz scheme keeps running in the background on its own clock. At each time $n$, we use the most recent estimates $\{\hat{p}_{ij}(n)\}$. Then $\hat{p}_{ij}(n) \to p_{ij} \ \forall \ i,j,$ as $n\uparrow\infty$. We can mimic our earlier analysis to obtain an arror bound
 $$E\left[\|x(n+1) - x^*\|^2\right] \leq \alpha E\left[\|x(n) - x^*\|^2\right] + \epsilon(n)$$
 where $\epsilon(n)$ is an asymptotically vanishing error variance term. This captures the combined effect of the quantities $\{$var$(\hat{p}_{ij}(n))\}$. Iterating, we see that we have
 $$E\left[\|x(n) - x^*\|^2\right] = O(\alpha^2 + \epsilon(n)) \to 0.$$\\

\item \textit{\large Tracking slowly varying rankings:}\\

 Our scheme can be modified to address the situation when the rankings drift slowly over time and the aim is to track them. We resort to the stochastic approximation version of the Kaczmarz method \cite{Thoppe}. Suppose a new observation is received for pair $(i,j)$ (assuming it is the one retained in our calculations, not $(j,i)$), then update $\hat{p}_{ij}(n)$ by a running average and concurrently run the constant step size asynchronous stochastic approximation scheme (with $m :=$ the lexicographical position of $(i,j)$ in our ordering)
\begin{eqnarray}
\lefteqn{x(n+1) =} \nonumber \\
&&\hspace{-.2in}(1 - cI\{\xi_n = m\})x(n) + cI\{\xi_n = m\}\times \frac{1}{\|a_m\|^2}\times \nonumber \\
\nonumber
&&\hspace{-.2in}\left(\sum_{\{j : (i,j) \in \mathcal{E}\}}( (y_{ij}(n) - y_{ji}(n)) - (x_i(n) - x_j(n)) )\right)a_i^T, \\ \label{Kac2}
\end{eqnarray}
where
$$y_{ij}(n) := - \log\left(\frac{1}{ \hat{p}_{ij}(n) } - 1 \right)$$
and $c > 0$ is a small\footnote{This should be small, but not so small that the algorithmic time scale given by $t(n) = nc$ is no faster than the time scale on which the environment changes, in which case the algorithm loses its tracking ability.} constant step size. There is, however, one subtlety. Earlier $\xi_n$ stood for the component we \textit{chose} to update, hence we could ensure that all components are sampled with a prescribed positive relative frequency. Now it is the component the environment chose to provide us data on. Thus we need to make assumptions regarding its statistics. One very general and convenient assumption is that the fraction of times any particular component was updated till time $n$ remains bounded away from zero with probability one as $n\uparrow\infty$. Then analysis similar to \cite{Thoppe} is possible, leading to the conclusion that the algorithm tracks the correct rankings with an error that is $O(c)$. This is what we expect from the theory of constant stepsize stochastic approximation, see \cite[Chapter 9]{borkar:book}.\\

Stochastic approximation is an incremental algorithm which uses decreasing step size to suppress the effect of discretization errors, noise and communication delays. It is unwarranted for our original set up because we have convergence even without incrementality which can only slow it down. In fact our experimentation did show degradation in speed of the original scheme (\ref{Kac}) when it was replaced by (\ref{Kac2}).\\
\end{enumerate}

\section{Experimental Results}
\ \\
In Section~\ref{Sec:ErrorAnalysis}, we studied the performance of our algorithm with respect to the normalized weight error $(\|w-\hat{w}\|)/\|w\|$. Since we are primarily concerned with the ranking and the error therein, we define the following error metric:
$$D_w(\sigma) = \sqrt{\frac{1}{2n\|w\|^2}\sum_{i<j}(w_i - w_j)^2\mathbb{I}((w_i - w_j)(\sigma_i - \sigma_j) > 0)}$$ where $\mathbb{I}(\cdot)$ is the indicator function, $w_i$'s are the actual weights of the players in the BTL model, and $\sigma$ the ordering according to the estimated weights. This error metric considers pairs of items and penalizes errors in their ordering, in proportion to the difference in their weights. Thus, the penalty is smaller if we get an error in ranking two players with similar weights, as compared to when they are vastly different. This error metric was also used in \cite{negahban2012rank} to evaluate the performance of their proposed ranking algorithm. Furthermore, \cite{negahban2012rank} showed that if $\hat{w}$ is the estimated weight vector used for the ordering $\sigma$, then
\begin{eqnarray}
\nonumber
D_w(\sigma) &\leq& \frac{\|w-\hat{w}\|}{\|w\|}.
\label{Eqn:DRank}
\end{eqnarray}

\textbf{Data Generation}: We consider $N = 400$ items and assign a weight $w_i$ to each item $i$ as $w_i = 10^{i/n}$. Thus, the dynamic range for the weights $b = \bar{w} / \underline{w} = 10$. For the underlying comparison graph $\G$, we assume an Erd\H{o}s-R\'enyi graph, so that for each pair of nodes, the edge between them exists with some probability $p$. Finally, we will denote the number of comparisons made per edge by $k$ and the total number of comparisons in $\G$ by $C$.

For each comparison, we randomly generate the output according to \eqref{Eqn:BTL}.  After collecting the outputs for all the comparisons, we run our proposed iterative algorithm, as described in Section~\ref{Sec:Problem}, and output the predicted weight vector upon convergence. We average our results over a large number of experiments and present the results below.

\subsection{Error Performance}

We compare the performance of the proposed Randomized Kaczmarz estimator with the Rank Centrality estimator from \cite{negahban2012rank} and the Maximum Likelihood Estimator for the BTL model from \cite{fordMLE}. We consider two error metrics, the normalized weight error $(\|w-\hat{w}\|)/\|w\|$ and $D_w(\sigma)$, as defined in \eqref{Eqn:DRank}. Figures~\ref{fig:Kvariation400Norm} and \ref{fig:Kvariation400D_w} illustrate the performance of the various algorithms in terms of the normalized weight error $(\|w-\hat{w}\|)/\|w\|$ and $D_w(\sigma)$ respectively, as a function of the number of comparisons per edge $k$ for a fixed value of edge probability $p \in \{0.16, 0.32\}$. The normalized weight error decays as $k^{-0.5}$, as expected from \eqref{Eqn:ERerror} in the error analysis section. In terms of $D_w(\sigma)$, which reflects the error in ranking the items, all the three estimators demonstrate very similar performance. 

Similarly, Figures~\ref{fig:Pvariation400Norm} and \ref{fig:Pvariation400D_w} show the dependence of these error metrics on the edge probability $p$, while fixing the number of comparisons per edge $k \in \{30, 100\}$. We can see that the error in the ordering matches almost exactly as that of the Rank Centrality, but the normalized weight error is a bit higher for higher edge probabilities. Hence, this shows that in terms of ordering, we perform as well as Rank Centrality (which in turn is as good as the Maximum Likelihood Estimator). Also we can see some dependence of the error on the edge probability too. However, some of the bounds used in the error analysis in Section~\ref{Sec:ErrorAnalysis} are too generous and hence the error bound in \eqref{Eqn:ERerror} fails to capture this dependence.

\begin{figure}[!t]
\centering
\begin{tikzpicture}
\begin{loglogaxis}[height=8cm,width=8cm,
	xlabel={Number of comparisons per edge $k$},
	ylabel={Normalised Weight Error}
]
\addplot+[black, mark = triangle*, smooth] coordinates {
(10,0.0873133)
(30,0.0615486)
(100,0.0298083)
(300,0.018745)
(1000,0.009849)
};

\addplot+[black, mark =triangle, smooth] coordinates {
(10,0.090072)
(30,0.0522295)
(100,0.0294848)
(300,0.0180182)
(1000,0.00989946)
};

\addplot+[black, mark =square, smooth] coordinates {
(10,0.089138)
(30,0.0522118)
(100,0.0288911)
(300,0.0179101)
(1000,0.00987114)
};

\addplot+[black, mark = *, smooth] coordinates {
(10,0.0703991)
(30,0.050109)
(100,0.0236533)
(300,0.0156695)
(1000,0.00817693)
};

\addplot+[black, mark =o, smooth] coordinates {
(10,0.0740889)
(30,0.0412814)
(100,0.022678)
(300,0.014855)
(1000,0.00805312)
};

\addplot+[black, mark =square*, smooth] coordinates {
(10,0.0738819)
(30,0.041261)
(100,0.0224991)
(300,0.0147814)
(1000,0.007949913)
};
\legend{\tiny{RK p = 0.16}, \tiny{RC p = 0.16}, \tiny{MLE p = 0.16}, \tiny{RK p = 0.32}, \tiny{RC p = 0.32}, \tiny{MLE p = 0.32}}\end{loglogaxis}
\end{tikzpicture}
\caption{Normalized Weight Errors in Randomized Kaczmarz (RK), Rank Centrality (RC) and Maximum Likelihood Estimator (MLE) for various number of comparisons per edge $k$ for a constant edge probability $p \in \{0.16, 0.32\}$.}
\label{fig:Kvariation400Norm}
\end{figure}
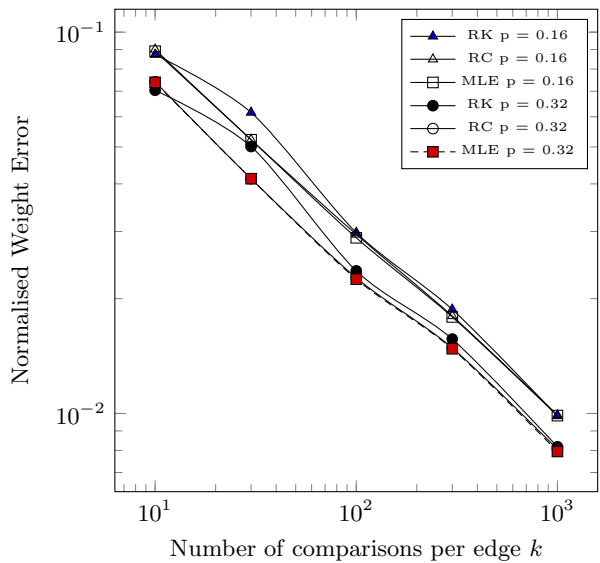

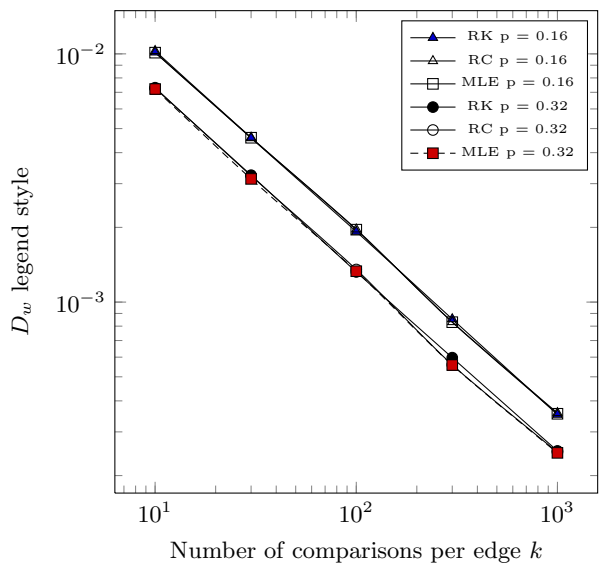
\begin{figure}[!b]
\centering
\begin{tikzpicture}
\begin{loglogaxis}[height=8cm,width=8cm,
	xlabel={Number of comparisons per edge $k$},
	ylabel={$D_w$}
	legend style={font=\small}
]
\addplot+[black, mark = triangle*, smooth] coordinates {
(10,0.010143)
(30,0.00458079)
(100,0.0019234)
(300,0.000854482)
(1000,0.000352248)
};

\addplot+[black, mark =triangle, smooth] coordinates {
(10,0.0102826)
(30,0.00460183)
(100,0.00196516)
(300,0.000831703)
(1000,0.000357676)
};

\addplot+[black, mark =square, smooth] coordinates {
(10,0.0101132)
(30,0.00459363)
(100,0.00195811)
(300,0.000829982)
(1000,0.000354899)
};

\addplot+[black, mark = *, smooth] coordinates {
(10,0.00726292)
(30,0.00324581)
(100,0.0013227)
(300,0.00059736)
(1000,0.00025143)
};

\addplot+[black, mark =o, smooth] coordinates {
(10,0.00730796)
(30,0.00322688)
(100,0.00135329)
(300,0.000559259)
(1000,0.000248729)
};

\addplot+[black, mark =square*, smooth] coordinates {
(10,0.00721912)
(30,0.003134956)
(100,0.00133192)
(300,0.000557718)
(1000,0.000246911)
};
\legend{\tiny{RK p = 0.16}, \tiny{RC p = 0.16}, \tiny{MLE p = 0.16}, \tiny{RK p = 0.32}, \tiny{RC p = 0.32}, \tiny{MLE p = 0.32}}
\end{loglogaxis}
\end{tikzpicture}
\caption{$D_w$ in Randomized Kaczmarz (RK), Rank Centrality (RC) and Maximum Likelihood Estimator (MLE) for various number of comparisons per edge $k$ for a constant edge probability $p \in \{0.16, 0.32\}$.}
\label{fig:Kvariation400D_w}
\end{figure}

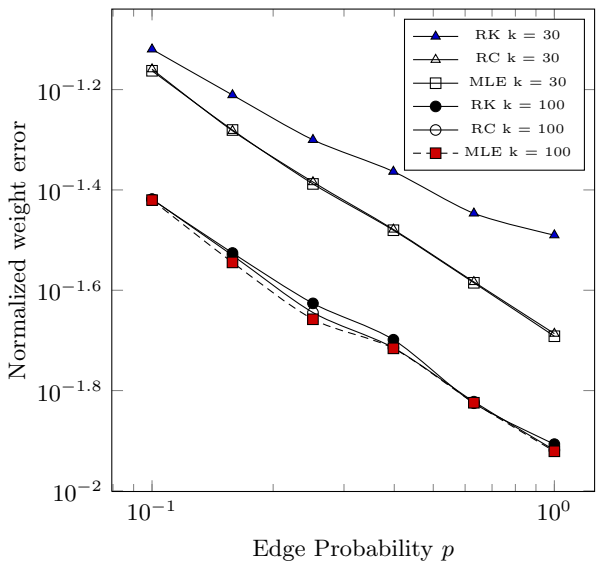
\begin{figure}[!t]
\centering
\begin{tikzpicture}
\begin{loglogaxis}[height=8cm,width=8cm,
	xlabel={Edge Probability $p$},
	ylabel={Normalized weight error}
]
\addplot+[black, mark = triangle*, smooth] coordinates {
(0.1,0.0759277)
(0.158489,0.0615486)
(0.251189,0.050109)
(0.398107,0.0433121)
(0.630957,0.0357569)
(1,0.0323538)
};

\addplot+[black, mark =triangle, smooth] coordinates {
(0.1,0.0693166)
(0.158489,0.0522295)
(0.251189,0.0412814)
(0.398107,0.0332082)
(0.630957,0.0261025)
(1,0.020587)
};

\addplot+[black, mark =square, smooth] coordinates {
(0.1,0.06879346)
(0.158489,0.0524165)
(0.251189,0.0409234)
(0.398107,0.0330912)
(0.630957,0.025997)
(1,0.0203417)
};

\addplot+[black, mark = *, smooth] coordinates {
(0.1,0.0381286)
(0.158489,0.0298083)
(0.251189,0.0236533)
(0.398107,0.0200189)
(0.630957,0.0149554)
(1,0.0123971)
};

\addplot+[black, mark =o, smooth] coordinates {
(0.1,0.0381728)
(0.158489,0.0294848)
(0.251189,0.022678)
(0.398107,0.0192915)
(0.630957,0.0150652)
(1,0.0120641)
};

\addplot+[black, mark =square*, smooth] coordinates {
(0.1,0.0380021)
(0.158489,0.0285446)
(0.251189,0.021989)
(0.398107,0.0192415)
(0.630957,0.0149892)
(1,0.0119951)
};
\legend{\tiny{RK k = 30}, \tiny{RC k = 30}, \tiny{MLE k = 30}, \tiny{RK k = 100}, \tiny{RC k = 100}, \tiny{MLE k = 100}}
\end{loglogaxis}
\end{tikzpicture}
\caption{Normalized Weight Error in Randomized Kaczmarz (RK), Rank Centrality (RC) and Maximum Likelihood Estimator (MLE) for various edge probability $p$ for a constant number of comparisons per edge $k \in \{30, 100\}$.}
\label{fig:Pvariation400Norm}
\end{figure}

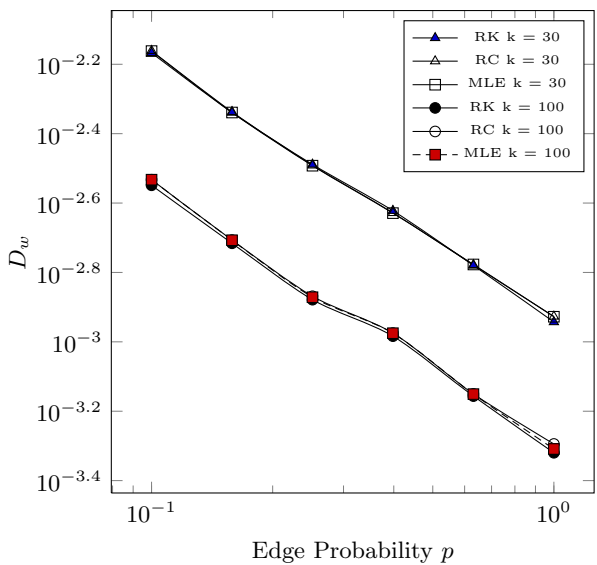
\begin{figure}[!b]
\centering
\begin{tikzpicture}
\begin{loglogaxis}[height=8cm,width=8cm,
	xlabel={Edge Probability $p$},
	ylabel={$D_w$}
]
\addplot+[black, mark = triangle*, smooth] coordinates {
(0.1,0.00680586)
(0.158489,0.00458079)
(0.251189,0.00324581)
(0.398107,0.00238729)
(0.630957,0.00166071)
(1,0.00114056)
};

\addplot+[black, mark =triangle, smooth] coordinates {
(0.1,0.00690162)
(0.158489,0.00460183)
(0.251189,0.00322688)
(0.398107,0.00235338)
(0.630957,0.0016663)
(1,0.00118349)
};
\addplot+[black, mark =square, smooth] coordinates {
(0.1,0.00688756)
(0.158489,0.00458130)
(0.251189,0.00321889)
(0.398107,0.00234975)
(0.630957,0.0016711)
(1,0.00118117)
};

\addplot+[black, mark = *, smooth] coordinates {
(0.1,0.0028261)
(0.158489,0.0019234)
(0.251189,0.0013227)
(0.398107,0.0010383)
(0.630957,0.000697524)
(1,0.00047892)
};

\addplot+[black, mark =o, smooth] coordinates {
(0.1,0.0029336)
(0.158489,0.00196516)
(0.251189,0.00135329)
(0.398107,0.00105998)
(0.630957,0.000707824)
(1,0.000507785)
};

\addplot+[black, mark =square*, smooth] coordinates {
(0.1,0.0029336)
(0.158489,0.00196327)
(0.251189,0.001344819)
(0.398107,0.00105921)
(0.630957,0.000707118)
(1,0.00049113)
};
\legend{\tiny{RK k = 30}, \tiny{RC k = 30}, \tiny{MLE k = 30}, \tiny{RK k = 100}, \tiny{RC k = 100}, \tiny{MLE k = 100}}
\end{loglogaxis}
\end{tikzpicture}
\caption{$D_w$ in Randomized Kaczmarz (RK), Rank Centrality (RC) and Maximum Likelihood Estimator (MLE) for various edge probability $p$ for a constant number of comparisons per edge $k \in \{30, 100\}$.}
\label{fig:Pvariation400D_w}
\end{figure}


\subsection{Stopping criterion and related issues}

\begin{enumerate}
\item \textit{\large Stopping Criterion:}

Define $\hat{w}_{n}$ to be the estimate of the weight vector after $n$ iterations.
We continue iterating till there is a very small change in the estimated weight vector over say $i$ iterations:
\begin{equation}
\label{Eqn:Stopping}
\frac{\|\hat{w}_{n+i} - \hat{w}_n\|}{\|\hat{w}_n\|} \leq \epsilon .
\end{equation}
Let the number of iteration required be denoted by $I_1(i,\epsilon)$. Figure~\ref{fig:stoppingVariation} shows the variation in $I_1(500,10^{-7})$ with different values of the edge probability $p$. The time complexity for the computation of this stopping criterion is $O(N)$ for each iteration, since we only need to calculate the norm.

\item \textit{\large Warm Start:}

One more possibility for speeding up the iterations is to initialize the iterative algorithm appropriately. We choose a reference node $i_{ref}$ (a good choice would be the distance centre of the graph) and assign the initial estimate for the reference node $v_{ref}^0 = 0$ . If the graph is connected, then there exists a path between each node and the reference node. Go along the shortest path for each node to assign a value $v_i^0 = \sum y_{ij}$ summed over the path, and use this rather than the zero vector as the initial condition. Let $I_2^k(i,\epsilon)$ be the number of iterations required to satisfy the criterion in \eqref{Eqn:Stopping}. Figure~\ref{fig:stoppingVariation} plots $I_2^{30}(500,10^{-7})$ vs the edge probability $p$ and also provides a comparison with $I_1(500,10^{-7})$. We can see that this choice of initial values helps reduce the number of iterations needed. The reduction in the number of iterations will be higher for higher number of comparisons, as we would be closer to the solution as $k$ increases. But there will be an initial computational cost of assigning these values before the iteration starts. Since this is similar to performing a Breadth-First search, the worst case time complexity for this pre-processing step is $O(|E|) = O(N^2p)$.

\item \textit{\large Convergence of $D_w(\sigma)$:}

As discussed before, since we are primarily interested in ranking items, the  $D_w(\sigma)$ error metric is more relevant than the normalized weight error. Here, we will calculate $D_w(\sigma)$ after each iteration and run iterations till it has converged. If this error metric has converged, then further iterations can only yield a better estimate of the weights, but the ranking will stay nearly the same. Let $D_w(\sigma^n)$ be the error in the ordering $\sigma^n$ after $n$ iterations. Then the iteration number $I_3(i,\epsilon) = \min\{n : D_w(\sigma^{n+i}) - D_w(\sigma^n) \leq \epsilon\}$. See Figure~\ref{fig:stoppingVariation} for a plot of $I_3(500, 10^{-7})$ and note that it is significantly smaller as compared to $I_1$ and $I_2$ which were based on the normalized weight error.

Note that to calculate $D_w(\sigma)$, we need to know the true ranking and for our experiments, we assume that to be true. The main goal of this experiment was to underscore the fact that convergence of ranks is much faster than the convergence of weight estimates. 

\item \textit{\large Top $K$ in $M$:}

Often it is not necessary to get the complete ranking correctly and it suffices to have the true $K$ top-ranked items to be among the estimated top $M$ items. Here we run our proposed algorithm with this as the stopping criterion. Again the knowledge of ground truth is necessary in this result, but it is an indicator of how fast this criterion is satisfied using this algorithm. Let $I_{4}(20,50)$ and $I_{4}(30, 75)$ denote the number of iterations needed to satisfy the stopping criteria top $20$ in top $50$ and top $30$ in top $75$ respectively, see Figure~\ref{fig:stoppingVariation} for an illustration. We can see that the top $20$ in $50$ criterion requires more iterations than top $30$ in $75$, which is expected as the former is a stricter criterion. Also, the top $K$ in $M$ criterion is achieved much faster than the other criterions. Further, the gap is the largest for high values of the edge probability $p$. This is because higher $p$ implies more edges in the network which in turn results in more weights being updated per iteration and thus the top $K$ start falling into the top $M$ sooner. As before, to implement such a stopping criterion we would need to know the true ranking and the main goal of the experiment was to demonstrate the faster convergence of ranks as opposed to weight estimates.


\end{enumerate}
Figure~\ref{fig:stoppingVariation} compares the number of iterations needed with the various stopping criteria discussed above. We have plots for the same synthesized data for: $I_1$, $I_2^{30}$, $I_3$, $I_{4}(20,50)$, and $I_{4}(30, 75)$.

\begin{figure}[!t]
\centering
\begin{tikzpicture}
\begin{axis}[height=8cm,width=8.5cm,
	xlabel={Probability},
	ylabel={Number of Iterations},
	cycle list name=black white
]
\addplot+[black, mark = o, smooth] coordinates {
(0.1,7387.1)
(0.158489,6154.1)
(0.251189,5802.7)
(0.398107,5268.5)
(0.630957,4759)
};

\addplot+[black, mark = triangle, smooth] coordinates {
(0.1,6777.8)
(0.158489,5947.8)
(0.251189,5188.8)
(0.398107,4803.7)
(0.630957,4153.7)
};
\addplot+[black, mark = x, smooth] coordinates {
(0.1,5651.1)
(0.158489,5454.3)
(0.251189,4976.4)
(0.398107,4417.2)
(0.630957,4082.6)
 };

\addplot+[black, mark = triangle*, smooth] coordinates {
(0.1,4070)
(0.158489,3382)
(0.251189,3182)
(0.398107,2578)
(0.630957,2225)
};

\addplot+[black, mark = *, smooth] coordinates {
(0.1,3206)
(0.158489,2345)
(0.251189,2245)
(0.398107,1796)
(0.630957,1847)
};

\legend{$I_1$, $I_2^{30}$, $I_3$, $I_{4}^{20/50}$, $I_{4}^{30/75}$}
\end{axis}
\end{tikzpicture}
\caption{Number of iterations required for the same data set by different stopping criteria.}
\label{fig:stoppingVariation}
\end{figure}
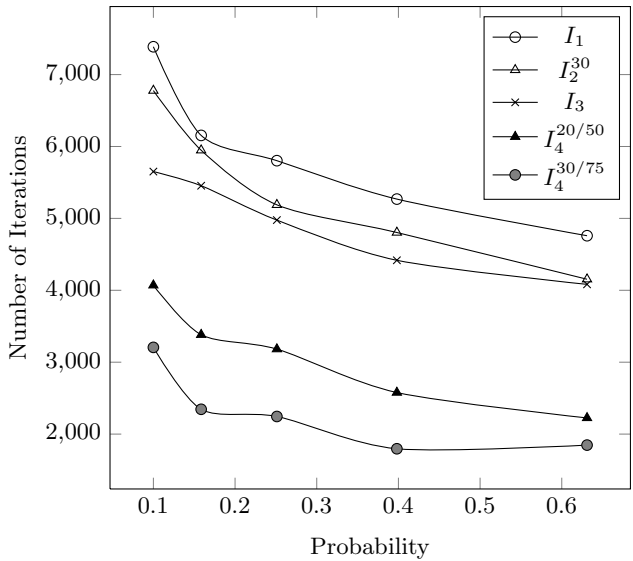

\begin{table*}[t]
\centering
\caption{Rank Aggregation for Tennis Players based on Pairwise Comparisons}
\label{table:tennis}
\begin{tabular}{@{}cccccccccc@{}}
\toprule
\multirow{2}{*}{ATP Rank} & \multirow{2}{*}{Name} & \multirow{2}{*}{Degree} & \multirow{2}{*}{Win Ratio} & \multicolumn{2}{c}{$\epsilon = 1$, RC} & \multicolumn{2}{c}{$\epsilon = 1$, RK} & \multicolumn{2}{c}{$\lambda = 0.05$, MLE} \\ \cmidrule(l){5-10} 
                          &                       &                         &                            & $w_i$              & Rank              & $w_i$              & Rank              & $w_i$                & Rank               \\ \midrule
1                         & N. Djokovic           & 88                      & 4.37                       & 2.09               & 2                 & 2.18               & 2                 & 2.13                 & 2                  \\
2                         & A. Murray             & 94                      & 3.03                       & 1.73               & 4                 & 1.79               & 4                 & 1.58                 & 4                  \\
3                         & R. Federer            & 88                      & 4.51                       & 2.15               & 1                 & 2.30               & 1                 & 2.07                 & 3                  \\
4                         & S. Wawrinka           & 96                      & 1.64                       & 1.14               & 9                 & 1.18               & 9                 & 1.04                 & 8                  \\
5                         & R. Nadal              & 96                      & 4.73                       & 1.95               & 3                 & 1.99               & 3                 & 2.15                 & 1                  \\
                                 \vdots                                   &                       \vdots&                         \vdots&                                                                      \vdots&                    \vdots&                   \vdots&                    \vdots&                   \vdots&                      \vdots&                    \vdots\\
150                       & B. Kavcic             & 103                     & 0.88                       & 0.56               & 87                & 0.56               & 85                & 0.58                 & 69                 \\
151                       & K. Khachanov          & 28                      & 0.35                       & 0.37               & 152               & 0.36               & 152               & 0.54                 & 110                \\
152                       & J. Nieminen           & 100                     & 0.81                       & 0.65               & 48                & 0.66               & 46                & 0.61                 & 52                 \\
153                       & J. Melzer             & 102                     & 0.99                       & 0.73               & 30                & 0.74               & 29                & 0.68                 & 30                 \\
154                       & J. Thompson           & 26                      & 0.35                       & 0.36               & 153               & 0.35               & 153               & 0.55                 & 108                \\ \bottomrule
\end{tabular}
\end{table*}

\subsection{Performance on real dataset}

We wanted to evaluate the performance of our proposed algorithm on real data with some head to head statistics where all players have not played each other. Professional lawn tennis seemed to be a good option. We collected the data of all the head to head matches of $154$ players as available in November 2015. We have a connectivity of 0.54 in the observed data with an average $2.5$ comparison per edge. Since there are pairs $i, j$ such that $i$ has won all its matches against $j$, if we used the fraction of wins as our probability estimate, as done in \eqref{Eqn:ProbEstimate}, it would yield $y_{ij} = \infty$ and $y_{ji} = -\infty$ for which our iterative algorithm would not work. Hence, a regularization was necessary and we redefine the equation for $\hat{p}_{ij}$ as: 
$$\hat{p}_{ij} = \frac{\sum X_{ij} + \epsilon}{\sum X_{ij} +\sum X_{ji} + 2\cdot \epsilon}$$
for some $\epsilon > 0$. Similarly, we also use a regularized version of the MLE which adds a penalty term of the form $\frac{1}{2}\lambda||\theta||^2$ to the objective function of the corresponding convex optimization problem, see \cite{negahban2012rank} for details.  Setting $\epsilon = 1, \lambda = 0.05$ and running the proposed algorithm, we get Table \ref{table:tennis}. We can see that if we had ranked the players by solely using the winning ratio, Rafael Nadal would have been ranked first. However, our algorithm also puts weight on the rank of the beaten opponent and this enables Roger Federer to grab the top position. Since we have taken all the played matches into account, there are differences with the current ATP rankings which only take recent performance into account. For example, the players who have performed well overall but not as good in the recent past like Jarkko Nieminen and Jurgen Melzer are ranked much higher then their current ATP rankings.

\section{Conclusions}
\ \\
We have considered the problem of rank aggregation of entities associated with the nodes of a connected graph when pairwise comparisons for neighboring nodes are available. Using the Bradley-Terry-Luce model, we associate preference probabilities in terms of certain node weights which are then to be estimated in order to come up with the overall ranking. Using a simple transformation, this is reduced to the problem of solving an underdetermined system of linear equations. We use the randomized Kaczmarz scheme for the purpose, which has provable convergence and exponential decay of mean square error, and in addition shows excellent performance in examples. We also discussed several variations, notably an online version. Further, we observed empirically that the rank order converges much faster than the weights themselves. Also, if one settles for the softer criterion of `top $K$ in top $M$' for prescribed $M > K$, again the convergence is very fast.\\

One of the future directions is to consider \textit{choosing} edges of the graph to sample comparative preferences on subject to a suitable cost of sampling, as also to come up with effective schemes when the sampled edges do not form a connected graph and in fact may form  a significantly small subset of the edge set. In addition, we plan to conduct more extensive numerical simulations as well as evaluations on real datasets in the future. 

\bibliographystyle{abbrv}
\bibliography{sigproc}

\appendix

We provide the proof of Theorem~\ref{Thm:RateK} here, similar analysis has also been done in \cite{Strohmer, Deanna1, gower2015randomized}. 

Since $x^*$ is a solution to $Ax = b$, we have
\begin{displaymath}
b_i - \langle a_i, x^*\rangle = 0 \ \forall \ i
\end{displaymath}
Then we have the following sequence of equations.
\begin{eqnarray*}
&&\hspace{-.2in}x(n+1) - x^* = (x(n) - x^*) - \sum_iI\{\xi(n) = i\}\times \\
&& \hspace{1.4in}\left(\frac{\langle a_i, x(n) - x^*\rangle}{\|a_i\|^2}\right)a_i^T \\
\Longrightarrow && \\
&&\hspace{-.2in} e(n+1) = e(n) - \sum_iI\{\xi(n) = i\}\langle\check{a}_i, e(n)\rangle\check{a}_i^T \\
\Longrightarrow && \\
&&\hspace{-.2in}\|e(n+1)\|^2 = \|e(n)\|^2 - \sum_iI\{\xi(n) = i\}|\langle \check{a}_i, e(n)\rangle|^2  \\
&&\hspace{.45in}= \|e(n)\|^2 \cdot \left(\!1 - \! \sum_iI\{\xi(n) = i\}|\langle \check{a}_i, \check{e}(n)\rangle|^2\right).
\end{eqnarray*}

\remove{
\noindent $\Longrightarrow$ \\
\begin{displaymath}
e(n+1) = e(n) - \sum_iI\{\xi(n) = i\}\langle\check{a}_i, e(n)\rangle\check{a}_i
\end{displaymath}

\noindent $\Longrightarrow$ \\
\begin{eqnarray*}
\|e(n+1)\|^2 &=& \|e(n)\|^2 - \sum_iI\{\xi(n) = i\}|\langle \check{a}_i, e(n)\rangle|^2  \\
&=& \|e(n)\|^2 \times \\
&&\left(1 - \sum_iI\{\xi(n) = i\}|\langle \check{a}_i, \check{e}(n)\rangle|^2\right).
\end{eqnarray*}
}
\noindent Taking expectation on both sides, we have
\begin{eqnarray*}
E\left[\|e(n+1)\|^2\right] &=& E\Big[ \|e(n)\|^2 \times \\
&&\left(1 - \sum_iI\{\xi(n) = i\}|\langle \check{a}_i, \check{e}(n)\rangle|^2\right)\Big]
\end{eqnarray*}
\begin{eqnarray*}
&=& E\Big[ E\Big[\|e(n)\|^2 \times \\
&&\left(1 - \sum_iI\{\xi(n) = i\}|\langle \check{a}_i, \check{e}(n)\rangle|^2\right) | e(n)\Big]\Big] \\
&=& E\Big[ \|e(n)\|^2\left(1 - \sum_ip_i|\langle \check{a}_i, \check{e}(n)\rangle|^2\right)\Big] \\
&\leq& E\Big[ \|e(n)\|^2 \times \\
&&\left(1 - \min_{\{s : \|s\| = 1, s \in H_0\}}\sum_ip_i|\langle \check{a}_i, s\rangle|^2\right)\Big] \\
&=& (1 - \lambda_{min}) \cdot E\left[\|e(n)\|^2\right].
\end{eqnarray*}
Since $\lambda_{min} \in (0, 1)$, the second claim follows. The first claim then follows from the Borel-Cantelli lemma, combined with Markov's inequality.



\end{document}